\documentclass{article}
\usepackage{iclr2026_conference,times}

\usepackage{amsmath,amsfonts,bm}

\def\eqref#1{equation~\ref{#1}}

\def\1{\bm{1}}

\DeclareMathAlphabet{\mathsfit}{\encodingdefault}{\sfdefault}{m}{sl}
\SetMathAlphabet{\mathsfit}{bold}{\encodingdefault}{\sfdefault}{bx}{n}

\usepackage{amsmath,amssymb,amsthm,mathtools}
\usepackage{booktabs,multirow,tabularx}
\usepackage{longtable}
\usepackage{graphicx}
\usepackage{float}
\usepackage{xcolor}
\usepackage{microtype}
\usepackage{enumitem}
\usepackage{tikz}
\usetikzlibrary{arrows.meta,positioning,fit}
\usepackage{hyperref}
\hypersetup{hidelinks}
\usepackage{url}

\title{COVER: Identifiable Evaluation of Coalition Routing}

\author{Raghul Sugumar \qquad Amrit Gopinath\\
Sri Sivasubramaniya Nadar College of Engineering}

\iclrfinalcopy

\newtheorem{theorem}{Theorem}
\newtheorem{proposition}{Proposition}

\newcommand{\method}{\textsc{COVER}}
\newcommand{\ptgst}{\textsc{Partition-COVER}}
\newcommand{\unified}{\textsc{Unified-COVER}}
\newcommand{\agents}{\mathcal{A}}
\newcommand{\tasks}{\mathcal{D}}

\newcommand{\regret}{\mathcal{R}}

\begin{document}
\maketitle

\begin{abstract}
When a multi-agent system changes its team, it also changes the messages and final answer it produces, so an end-to-end accuracy gap does not by itself identify a routing effect. We introduce \method{}, an evaluation contract that fixes a public information boundary, downstream stack $G$, and finite legal team family before outcomes are generated. Complete coverage identifies exact \emph{finite-benchmark} oracle regret conditional on that stack. For any finite collection of frozen policies, executing the union of their distinct selected teams is the minimal assumption-free support for every pairwise policy contrast, though not for absolute oracle regret. Two controlled tables with source-ID-disjoint splits test the instrument. On MuSiQue-12, a pre-specified \emph{privileged positive control} improves regret from 0.532 to 0.402; a later public-interface control reaches 0.424 versus 0.554 but is retrospective. On HotpotQA-4, a pre-specified public direct scorer improves regret from 0.313 to 0.110. In fixed-stack Llama execution, verified route regret improves by 0.190, while the raw-answer gain is 0.010 with an interval crossing zero. A five-family ToolSandbox variant-shift validation exhaustively evaluates 16 declared teams on 14 untouched task variants (224/224 valid rows): the declared-family oracle reaches 0.768 safe-evidence completion, while the prospectively frozen router gets 0.637 (regret 0.131), failing the predeclared 0.10 criterion. A later retrospective comparator reaches 0.655, matching all-workers with 4.57 versus 5.00 workers on average. Thus COVER exposes selection headroom without manufacturing a routing win. A crossed-stack diagnostic shows absolute scores depend on $G$ but finds no detectable router-by-finalizer interaction. COVER is an auditable measurement methodology, not a claim of stack-invariant or universal agent-routing superiority.
\end{abstract}

\section{The problem in one picture}
\label{sec:intro}

Suppose two multi-agent systems choose different workers for the same question. If one answers correctly more often, was its \emph{selection} better? Not necessarily: changing a team also changes private messages and can change how a finalizer synthesizes them. This is the basic confound in routing evaluation.

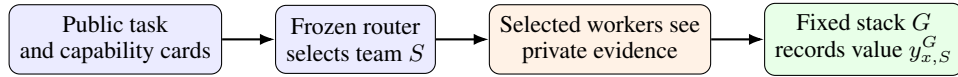
\begin{figure}[H]
\centering
\begin{tikzpicture}[font=\small,box/.style={draw,rounded corners,align=center,minimum height=9mm,inner sep=4pt},pub/.style={box,fill=blue!8},hid/.style={box,fill=orange!10},resultbox/.style={box,fill=green!9},arr/.style={-{Latex[length=2mm]},thick}]
\node[pub] (input) {Public task\\and capability cards};
\node[pub,right=7mm of input] (route) {Frozen router\\selects team $S$};
\node[hid,right=7mm of route] (work) {Selected workers see\\private evidence};
\node[resultbox,right=7mm of work] (final) {Fixed stack $G$\\records value $y_{x,S}^{G}$};
\draw[arr] (input)--(route); \draw[arr] (route)--(work); \draw[arr] (work)--(final);
\node[below=5mm of route,align=center] {Repeat for every legal $S\in\mathcal F_x$;\\then compare all routers on the \emph{same} table.};
\end{tikzpicture}
\caption{COVER changes only the selected coalition. Orange information is unavailable when the route is chosen.}
\label{fig:contract}
\end{figure}

\method{} makes the counterfactual explicit. For a task $x$, let $\agents_x$ be its workers, $\mathcal F_x\subseteq2^{\agents_x}$ its legal teams, and $G$ the frozen stack comprising public/private timing, workers, message order, prompt and token policy, finalizer, judge, and utility. Let $y_{x,S}^{G}\in[0,1]$ be the outcome of executing exactly $S$ under $G$. We call
\begin{equation}
 V_x^{\mathrm{bench},G}(S):=y_{x,S}^{G},\qquad \regret^{\mathrm{bench},G}(r)=\frac1{|\tasks|}\sum_{x\in\tasks}\left[\max_{S\in\mathcal F_x}y_{x,S}^{G}-y_{x,r(x)}^{G}\right]
 \label{eq:regret}
\end{equation}
the finite-benchmark value and regret of router $r$ under $G$. Lower regret is better. Exactness is within-stack: it neither asserts that team rankings are invariant to replacing $G$ nor identifies stochastic deployment utility. The latter, $V_x^{\mathrm{pop},G}(S)=\mathbb E_\omega[y_{x,S,\omega}^{G}]$, requires repeated executions and is not identified by one provider call.

If two systems change their workers or finalizer as well as their route, the selection contrast is not identified: the same observed answers are compatible with a routing effect or with no routing effect and a changed downstream pipeline. In contrast, when all legal teams are observed under one fixed protocol, Equation~\ref{eq:regret} is a table lookup---no model of unseen coalitions is needed. If teams are missing, one may report sampled-action regret, but not an exact oracle claim without additional assumptions.

The complete-table lookup is not itself the theoretical contribution. The substantive design question is how much intervention support is necessary when full enumeration is unaffordable. COVER gives a sharp answer for frozen-policy comparison: the route union is minimal for relative contrasts, whereas the complete family is required for assumption-free absolute oracle regret. The experiments instantiate both regimes and show when confusing them changes a conclusion.

\paragraph{What we contribute.} We contribute (i) a stack-conditional evaluation contract and a minimal-support theorem separating absolute oracle regret from relative frozen-policy comparison; (ii) two controlled tables with source-ID-disjoint train/development/held-out splits; (iii) a fixed-stack execution decomposition; and (iv) a natural heterogeneous-tool validation that succeeds as a measurement while rejecting its frozen routing headline. COVER is a methodology paper. The routers are validation instruments, not a new universal architecture.

\paragraph{Reader's guide.} The main narrative is deliberately linear: Section~\ref{sec:method} identifies the estimand; Section~\ref{sec:experiments} evaluates two complete controlled tables; Section~\ref{sec:results} decomposes execution under one fixed stack; and Section~\ref{sec:natural} tests the endpoint in ToolSandbox. All other experiments are visibly supplemental.

\section{Related work: what COVER is and is not}

Model routers choose a complete model or scaffold using preference, similarity, or confidence \citep{ong2024routellm}. Multi-agent systems additionally choose roles, communication paths, and workers \citep{wu2023autogen,du2023debate,wang2024moa}. Their usual end-to-end comparison is valuable for product evaluation, but it need not answer the narrower question of coalition-selection quality under a common downstream protocol.

A contemporary post-audit example, Pandora's Router, asks whether the expected benefit of purchasing a more accurate specialist-value estimate justifies its inspection cost \citep{fisch2026pandora}. Pandora optimizes information acquisition and specialist allocation; COVER instead asks what coalition-selection contrast is identified once the information boundary and downstream stack are declared. Pandora v2 postdates the audit freeze and is therefore cited here without altering the frozen 30-paper counts.

\paragraph{A targeted audit of what ``routing'' denotes.} A frozen, purposive, non-systematic audit of 30 recent papers documents within this sample that ``routing'' spans joint-system design, communication topology, sequential policies, and team selection. These are different estimands, not a catalogue of invalid work. The coding rules, counts, single-coder limitation, and full paper-level table are confined to Appendix~\ref{app:routing-audit}. COVER targets the narrower case of a declared finite coalition family under one common downstream protocol.

Set encoders, submodular selection, DPPs, and correlation clustering provide useful mechanisms for scoring groups \citep{zaheer2017deepsets,lee2019settransformer,kulesza2012dpp,bansal2004correlation}. We use them only as validation-policy ingredients. COVER does not claim a new attention primitive, set-function class, or optimizer. Its target is instead the quantity a routing experiment has identified.

Finally, factorial and off-policy methods estimate contrasts for actions that are not all observed \citep{dasgupta2015factorial,rebello2023factored,shimizu2024combinatorial}. COVER occupies the complementary regime: a small, declared action family can be executed completely. This costs more up front, but makes the finite benchmark oracle observable rather than estimated under a surrogate model.

ToolLLM retrieves candidate APIs before use \citep{qin2023toolllm}, while DyLAN performs preliminary team optimization before task solving \citep{liu2023dylan}. Our natural-study comparators adapt these two selection ideas to a frozen finite coalition table. They are deliberately labeled ``ToolBench-style'' and ``DyLAN-style'': table lookup cannot reproduce either published system's training, prompting, or execution stack, and we do not present the adaptations as head-to-head replications.

\section{The evaluation contract}
\label{sec:method}

For every task, a valid COVER comparison fixes five things:
\begin{enumerate}[leftmargin=*,itemsep=2pt,topsep=2pt]
\item \textbf{Action space.} Declare which teams are legal before outcomes are generated.
\item \textbf{Information boundary.} The router receives task text and public cards; evidence, answers, required-team labels, and outcomes stay hidden.
\item \textbf{Downstream protocol.} Workers, ordering, communication, utility, and finalizer are identical across routers.
\item \textbf{Coverage.} Execute every legal team for absolute oracle regret, or the declared route union for relative frozen-policy contrasts.
\item \textbf{Inference.} Freeze routers before held-out comparison and resample tasks, not correlated worker pairs, for uncertainty.
\end{enumerate}

This contract is complementary to off-policy evaluation and learned surrogate approaches: those extrapolate to unobserved actions under assumptions; we instead evaluate a small declared action space completely. It is also different from model routing, which typically chooses a single model or a full scaffold rather than a coalition of evidence workers \citep{ong2024routellm,zaheer2017deepsets,lee2019settransformer}.

\subsection{Identification hierarchy and minimal support}

The distinction between a table and a deployment distribution matters. Let $O_x\subseteq\mathcal F_x$ be the teams that were actually executed under $G$ and let $m_x^G=\max_{S\in O_x}y_{x,S}^G$. If values lie in $[0,1]$, then for any router whose chosen team was observed,
\begin{equation}
 m_x^G-y_{x,r(x)}^G\ \leq\ \regret_x^{\mathrm{bench},G}(r)\ \leq\ 1-y_{x,r(x)}^G.
 \label{eq:bounds}
\end{equation}
If $O_x\subsetneq\mathcal F_x$, both endpoints are attainable by finite tables that agree on every observed outcome: assign every unseen team a value at most $m_x^G$ for the lower endpoint, or assign one unseen team value one for the upper endpoint. When $O_x=\mathcal F_x$, regret is exactly $m_x^G-y_{x,r(x)}^G$. Thus complete coverage is the generic route to exact benchmark regret under the declared $G$; with partial coverage, the only assumption-free exception is saturation at $m_x^G=1$. Structural, smoothness, or generative assumptions can support estimates with partial coverage, but should be reported as such.

Absolute regret is not always the required target. Let $\mathcal R=\{r_1,\ldots,r_K\}$ be frozen policies evaluated under the same $G$, and let $A_x(\mathcal R)=\{r_k(x):1\leq k\leq K\}$ be their distinct selected teams on task $x$.

\begin{theorem}[Minimal support for frozen-policy comparison]
\label{thm:minimal-support}
Suppose finite outcomes are otherwise unrestricted in $[0,1]$. If $|A_x(\mathcal R)|\geq2$, every pairwise contrast
\[
\Delta_x^G(r_i,r_j)=y_{x,r_i(x)}^G-y_{x,r_j(x)}^G
\]
is point-identified from execution support $O_x$ if and only if $A_x(\mathcal R)\subseteq O_x$. Consequently, the per-task route union is the unique inclusion-minimal assumption-free design for all pairwise contrasts and requires exactly $|A_x(\mathcal R)|$ executions. If all policies select the same team, every contrast is identically zero without execution. In neither case does relative identification recover absolute oracle regret unless the oracle is otherwise identified.
\end{theorem}

Sufficiency is direct lookup. For necessity, omit any selected team and two bounded tables can agree on every observed intervention while assigning different values to that team, changing at least one contrast with a policy selecting another team. The taskwise result also identifies averages, win/tie/loss counts, and policy rankings over a fixed evaluation set. Appendix~\ref{app:formal} gives the full proof and separates this result from moving-stack comparisons.

\begin{table}[t]
\centering
\small
\caption{Information visible at each point in a COVER evaluation.}
\label{tab:boundary}
\begin{tabularx}{\linewidth}{lXX}
\toprule
Stage & Available & Deliberately unavailable \\
\midrule
Route selection & task text, public cards, declared budget & private evidence, answers, outcome table, required team \\
Worker execution & selected worker's private evidence & other workers' private evidence \\
Finalization & canonical selected messages & router identity and unselected messages \\
Offline scoring & frozen table values and route files & any held-out training outcome \\
\bottomrule
\end{tabularx}
\end{table}

For genuine execution, full coverage is usually too expensive. Theorem~\ref{thm:minimal-support} justifies executing the deduplicated union of teams selected by the frozen policies under one protocol. This design identifies their paired differences with no model for unexecuted teams while making no execution-oracle claim. It is the estimand used below.

\begin{table}[t]
\centering
\small
\caption{COVER's compute/identification regimes. Costs are interventions per task; only the first regime identifies absolute finite-benchmark oracle regret.}
\label{tab:regimes}
\begin{tabularx}{\linewidth}{l r X}
\toprule
Coverage regime & Interventions/task & Identified quantity \\
\midrule
Complete finite family & all legal teams & exact finite-benchmark oracle regret \\
Precommitted sampled family & declared sampled teams & sampled-action regret or bounds, not an oracle \\
Frozen route union & distinct selected teams & all pairwise policy contrasts; inclusion-minimal \\
\bottomrule
\end{tabularx}
\end{table}

\section{Two controlled intervention tables}
\label{sec:experiments}

We instantiate the contract twice in multi-hop QA. The task and public cards are visible to a router; each selected worker receives only its own private evidence; a fixed symbolic finalizer scores whether the construction-defined required evidence is present. Source IDs are disjoint across train, development, and held-out sets. A streaming audit confirms exactly one positive team per held-out task in both binary tables (300/300 HotpotQA tasks and 500/500 MuSiQue tasks). Exhaustive regret is therefore exact for these constructed tables, but the outcome is a controlled evidence-selection measurement, not natural team utility or a claim about unrestricted agents.

\begin{table}[t]
\centering
\small
\caption{The two intervention-complete controlled instances. A task, rather than a card pair or intervention, is the unit of inference.}
\label{tab:benchmarks}
\begin{tabularx}{\linewidth}{l r r X}
\toprule
Instance & Held-out tasks & Legal teams/task & Construction purpose \\
\midrule
MuSiQue-12 & 500 & 220 size-three teams & Tests globally coherent components amid distractor components \\
HotpotQA-4 & 300 & 4 size-three teams & Tests direct team scoring without a partition decoder \\
\bottomrule
\end{tabularx}
\end{table}

\subsection{Policies are validation instruments}

\unified{} is a permutation-invariant direct scorer. It encodes the task and public cards with a frozen sentence encoder, forms task-conditioned card states using multiplicative and distance features, pools both the full worker set and a proposed team, and scores the team with a small MLP. It is trained on observed training-coalition outcomes using regression, ranking, and best-team losses. No positional worker identifier is provided.

\ptgst{} is intentionally different. It ranks a triple using a fixed local compatibility proxy and the quality of its best balanced partition completion, computed exactly by bitmask dynamic programming. Its pre-specified MuSiQue interface is trained on train-only latent component labels. We therefore classify it as a \emph{privileged positive control}, not a candidate deployment router: conditional on the same privileged interface, the matched graph contrast asks only whether the table detects structured global completion. A separate public interface learns compatibility from public text and training intervention outcomes; that control is closer to deployment, but retrospective.

\subsection{Pre-specified analysis}

The original confirmatory family comprises the MuSiQue privileged positive-control contrast, the HotpotQA public-interface comparison, and the Llama verified-evidence execution endpoint. We use paired task bootstrap intervals, two-sided sign-flip tests, and Holm correction within that family. ``Confirmatory'' here describes the frozen test, not deployment validity: the MuSiQue contrast confirms instrument sensitivity conditional on privileged supervision. ToolSandbox is a later natural validation. All other endpoints, controls, scale studies, comparators, provider replications, and failed or null gates are secondary and labeled as such; most appear only in the supplement and append-only ledger, while the retrospective public MuSiQue control and ToolSandbox comparator rows remain in the main text to expose the deployment-relevant comparison and make the natural null interpretable.

\begin{table}[t]
\centering
\small
\caption{Controlled findings and their evidential status. Regret is exact over each declared offline family conditional on its frozen stack; exact scoring does not make a privileged policy deployable or a retrospective policy confirmatory.}
\label{tab:main-results}
\begin{tabularx}{\linewidth}{l r r X}
\toprule
Study & Tasks & Result & Interpretation \\
\midrule
MuSiQue-12 public & 500 & 0.424 vs 0.554 & Deployment-relevant public interface; retrospective control \\
MuSiQue-12 privileged & 500 & 0.402 vs 0.532 & Pre-specified positive control of instrument sensitivity \\
HotpotQA-4 & 300 & 0.110 vs 0.313 & Direct set scoring is a useful bias in a four-way decision \\
\bottomrule
\end{tabularx}
\end{table}

\paragraph{MuSiQue-12: public control and privileged positive control.} Each pool has twelve permuted cards and 220 legal triples. The deployment-facing result is the public interface learned only from public task/card text and training intervention outcomes: partition decoding obtains 0.424 regret versus 0.554 for its matched graph, with no latent membership labels at training or test time. Because this interface was designed after the primary analysis, the result is retrospective and cannot replace a prospectively frozen public-policy validation. Separately, the pre-specified interface trained from latent component membership obtains 0.402 versus 0.532. Those labels are privileged benchmark-construction information. We retain that pair solely as a positive control showing that the outcome table is sensitive to global completion when the relevant local structure is supplied; it is not evidence for an externally useful router. Dependence, no-partition, and representation controls are reported in Appendix~\ref{app:results}.

\paragraph{HotpotQA-4: a direct-comparator stress test.} Each task has four workers and one must choose a size-three team. The direct baseline family is deliberately non-degenerate: a balanced multinomial leave-one-out classifier and an unrestricted two-hidden-layer team MLP both consume exactly the public task/card text and training-only best-full-team labels, and both can represent every one of the four decisions. Neither reads required roles, evidence, answers, semantic-route labels, or held-out outcomes. The leave-one-out classifier is the stronger held-out direct baseline (regret 0.3133; the MLP is 0.3500 at seed 29), so the reported \unified{} contrast is conservative: 0.1100 versus 0.3133, paired gain 0.2033 (CI [0.1533, 0.2567], sign-flip $p=.00001$; 69/223/8 wins/ties/losses). Low-order unary/pairwise baselines are degenerate under the strict proof gate because incomplete teams have value zero; they are therefore not the main comparator. This is evidence for an inductive bias and sample efficiency in a four-action decision---not evidence that direct set scoring already solves large coalition routing.

\begin{figure}[t]
\centering
\begin{tikzpicture}[x=.58cm,y=.60cm,font=\small,bar/.style={fill=blue!55},label/.style={anchor=east}]
\draw[->] (0,0)--(8,0) node[right]{mean benchmark regret (lower is better)};
\node[label] at (0,1.2) {MuSiQue graph}; \fill[gray!45] (0,0.9) rectangle (5.54,1.5); \node[right] at (5.54,1.2) {0.554};
\node[label] at (0,2.2) {MuSiQue partition}; \fill[bar] (0,1.9) rectangle (4.24,2.5); \node[right] at (4.24,2.2) {0.424};
\node[label] at (0,3.6) {Hotpot leave-one-out}; \fill[gray!45] (0,3.3) rectangle (3.13,3.9); \node[right] at (3.13,3.6) {0.313};
\node[label] at (0,4.6) {Hotpot direct set}; \fill[bar] (0,4.3) rectangle (1.10,4.9); \node[right] at (1.10,4.6) {0.110};
\end{tikzpicture}
\caption{Controlled public-interface comparisons. The MuSiQue pair is retrospective; the HotpotQA pair is pre-specified. Regret is exact within each declared action space; scales are not compared across instances.}
\label{fig:offline}
\end{figure}
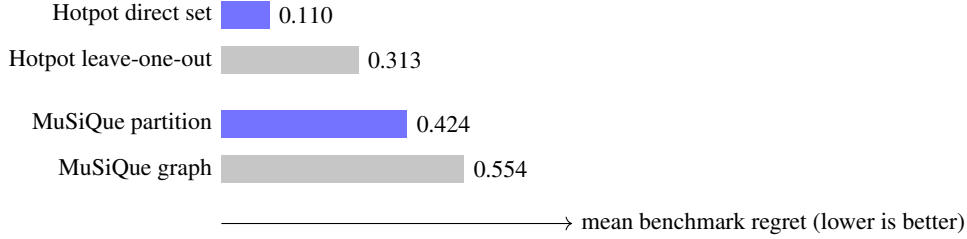

\section{Execution decomposition under a fixed stack}
\label{sec:results}

Offline tables isolate selection exactly, but genuine worker messages are still useful as a diagnostic. We freeze both routes, execute their union with the same hosted workers and fixed finalizer, then compare only the selected routes. The oracle term cancels in this paired difference, so this identifies a relative route comparison even though it does not recover an execution oracle.

The pre-specified 300-task Llama study makes the key distinction visible: \unified{} transports verifier-complete evidence more often, but its raw exact answer advantage over leave-one-out is only one percentage point and its 95\% interval crosses zero. A better route can therefore be hidden by a weak finalizer; conversely, an end-to-end answer gain alone cannot prove a routing effect without the fixed-protocol control.

\begin{table}[H]
\centering
\small
\caption{Frozen genuine execution. Executed-route regret is calculated only on the union of the frozen selected routes, not over all possible teams.}
\label{tab:genuine}
\begin{tabular}{lrrrr}
\toprule
Family & Router & \shortstack{Verified evidence\\complete} & \shortstack{Raw answer\\exact} & \shortstack{Raw route\\regret} \\
\midrule
Llama & \unified{} & 81.67\% & 36.33\% & 0.050 \\
Llama & leave-one-out & 62.67\% & 35.33\% & 0.060 \\
\bottomrule
\end{tabular}
\end{table}

The verified evidence-complete Llama contrast is 0.190 executed-route-regret points (95\% CI [0.140, 0.240]; Holm $p=.00003$ within the original family). This endpoint records successful evidence/citation transport under the enforced verifier; it is not finalizer exact-answer success. The untouched raw-answer contrast is 0.010 (CI [$-0.0067$, 0.0267]). These numbers make the claim hierarchy concrete: verified evidence transport is the confirmatory execution result; raw-answer gains are conditional on the downstream configuration $G$.

That conditioning is substantive. The completed crossed $2\times2$ diagnostic shows that absolute execution values depend on the finalizer, while finding no detectable router-by-finalizer interaction in this two-finalizer panel (Table~\ref{tab:crossed-finalizer}; Appendix~\ref{app:results}). All four within-cell route advantages are positive, so the experiment does not demonstrate a routing-rank reversal; that remains a theoretical possibility under Proposition~1. The broader six-configuration panel remains incomplete. All provider replications and later follow-ups are disclosed in Appendix~\ref{app:results} and the ledger, but do not enter the main claim.

\section{Natural heterogeneous-tool validation}
\label{sec:natural}

The controlled QA tables identify their finite estimands, but their outcomes are constructed evidence coverage. We therefore ask whether the same protocol remains informative in a natural tool environment. The final ToolSandbox study exposes five heterogeneous public families---settings, date/utility, directory, messaging, and reminders---backed by official local tools. Its declared action family contains the five singletons, ten pairs, and the full five-worker team: 16 coalitions per task. A pinned Qwen2.5-14B-AWQ executor starts from fresh simulator state, and value is the fraction of official safe tool-state milestones completed. This decomposed endpoint attributes selection and execution separately from answer synthesis.

The development matrix combines 34 earlier tasks and contains all $34\times16=544$ outcomes. We then freeze a public capability router and evaluate an outcome-untouched 14-task within-family variant-shift validation. It is not scenario-family-disjoint replication. Every one of the $14\times16=224$ expected held-out rows is valid, and every task exhibits at least two distinct coalition values. The natural gate therefore passes its two prerequisite questions: the stack can use the tools, and changing the coalition changes the measured outcome.

\begin{table}[H]
\centering
\small
\caption{Five-family ToolSandbox variant-shift validation (14 tasks, 16 declared coalitions/task). Completion is the official safe-evidence endpoint; lower regret is better. ``Style'' rows are interface-level adaptations, not exact published-system reproductions.}
\label{tab:toolsandbox}
\begin{tabular}{lrrr}
\toprule
Policy or reference & Completion & Regret & Mean workers \\
\midrule
Per-task declared-family oracle & .768 & --- & --- \\
All workers & .655 & .113 & 5.00 \\
Development-best fixed team & .655 & .113 & 5.00 \\
Learned value router & .655 & .113 & 4.57 \\
ToolBench-style card retriever & .655 & .113 & 5.00 \\
Frozen public capability router & .637 & .131 & 3.00 \\
DyLAN-style trial selector & .482 & .286 & 2.50 \\
\bottomrule
\end{tabular}
\end{table}

The prospectively frozen capability router reaches 0.637 completion and regret 0.131 (95\% task-bootstrap CI [0.042,0.244]), failing the predeclared regret threshold of 0.10 and trailing all-workers by 0.018. The 0.10 criterion was fixed as a practical tolerance on the $[0,1]$ endpoint: a successful router could leave at most ten percentage points of mean safe-evidence completion below the declared-family oracle. This is a valid negative held-out result, not a failed run: all rows are present and the 0.113 gap between the declared-family oracle and all-workers shows genuine team-dependent headroom.

After the held-out summary had been inspected, we ran a no-API comparator suite on the already collected matrix. Every route is fitted using development outcomes only, serialized, and then scored by held-out-table lookup. The strongest five-fold-selected learned value router matches all-workers exactly at 0.655 while selecting the full team on 12 tasks and a pair on two, reducing mean team size from 5.00 to 4.57 without improving value. The development-best fixed team and ToolBench-style card retriever both reduce to all-workers; the DyLAN-style trial selector performs worse. None clears regret 0.10. Because the suite itself was specified after inspection of the held-out summary, only the comparator-suite rows in Table~\ref{tab:toolsandbox} are retrospective, even though their route fitting never reads held-out outcomes. They cannot retroactively rescue or replace the prospectively frozen router result.

This is the natural study's substantive finding: exhaustive evaluation reveals a 0.768 declared-family ceiling, but the later best development-frozen comparator leaves 0.113 value on the table. COVER succeeds here by distinguishing \emph{measurement headroom} from \emph{achieved routing performance}. A final-answer-only benchmark would additionally mix this gap with synthesis failures; a top-$k$ success metric would hide it behind ties. The decomposed finite table makes the unresolved routing problem visible.

\section{Artifacts and reproducibility}

Every controlled study is tied to a frozen source-ID manifest, held-out route file, task-level result file, and analysis command. The repository provides a sanitized intervention archive containing hashed task identities, generic coalition identifiers, and deterministic values. This archive is sufficient to re-score a supplied route file, but not to run a new public-information router by itself because source text and public cards are omitted for licensing reasons. To develop a new router, a researcher obtains the original datasets under their licenses, reconstructs the allowed public inputs with the released manifest builder, emits routes keyed by the frozen task hashes, and scores them against the archive.

The final ToolSandbox release contains the five-family development and held-out manifests, 544 development and 224 held-out task--coalition rows, the prospectively frozen capability-router decision file, and the retrospective offline baseline suite. The latter serializes its own method parameters and hashes before opening the held-out table. The released summary, freeze record, and 84 per-policy held-out routes are sufficient to recompute Table~\ref{tab:toolsandbox} without another model or API call. The earlier three-family 105-row table is retained in the study ledger as a capability pilot, not silently combined with or substituted for the final validation.

The hosted execution traces are not claimed to be perfectly replayable: a provider can change despite temperature-zero requests. We release compact per-task outcomes, prompts, route decisions, model identifiers, protocol hashes, and no-API commands for every calculation that does not require the provider. An append-only study ledger lists positive, null, failed-gate, diagnostic, incomplete, and infrastructure-invalid studies with their stopping rules. This is why the paper separates a finite deterministic table result from a stochastic population-utility claim.

\section{Discussion: when should one use COVER?}

COVER is most useful when a researcher can afford exhaustive outcomes for a meaningful but small team family, and wants to know whether a routing rule rather than a changed pipeline caused a measured effect. It is particularly suitable for evaluating a new selector against an established worker and finalizer stack. It is not the right tool for claiming exact oracle performance among hundreds of unconstrained tools. In that setting, a researcher should precommit a public sampling design, record inclusion probabilities, report sampled-action regret and its coverage, and test the sampler against full enumeration on a smaller overlap cohort.

Large spaces need not make every comparison unidentifiable: Theorem~\ref{thm:minimal-support} shows that the precommitted route union is exactly the minimal assumption-free support for all contrasts among the frozen policies, while leaving absolute oracle regret unidentified. Table~\ref{tab:regimes} states the general distinction; the stress-test implementation is supplemental.

The methodology also suggests a reporting norm. A paper should say whether its conclusion concerns (a) a finite controlled selection table, (b) evidence transport under fixed generation, or (c) final task success under stochastic execution. These are related but not interchangeable quantities. In our own evidence ladder, the Llama experiment demonstrates precisely why the distinction changes the conclusion.

The ToolSandbox result adds a second reporting norm: a declared-family oracle ceiling is not router performance. Here 0.768 is obtainable only by post-hoc per-task lookup over the 16 declared teams, whereas the prospectively frozen router scores 0.637 and a later retrospective development-only comparator scores 0.655. Reporting the oracle without achieved performance would turn treatment variation into a fictitious routing result; reporting only achieved performance would conceal that the environment contains exploitable heterogeneity. COVER requires both.

\section{Scope and conclusion}
\label{sec:conclusion}

COVER answers one narrow question well: within a declared finite action space and frozen stack, how much benchmark value does a router lose relative to the best team when only the selected coalition changes? Complete interventions, a strict information boundary, and frozen downstream execution make that question identifiable. The evidence spine is deliberately short: two controlled tables establish exact selection evaluation, one execution study separates evidence transport from synthesis, and the five-family ToolSandbox variant-shift validation establishes natural treatment variation while rejecting the frozen router headline.

The limits are equally important. The controlled outcomes are construction-defined; the pre-specified MuSiQue policy is only a privileged positive control, while its more deployment-relevant public-interface result is retrospective; the natural validation has only 14 within-family task variants and its retrospective baseline suite cannot be promoted to a prospective result; and a one-draw provider execution is not population utility. No result establishes stack-invariant, large-pool, or general agent-system superiority. The headline is methodological: COVER can identify a routing gain, separate it from synthesis, and---when the data demand it---show that substantial declared-family oracle headroom has not yet been converted into deployable routing performance. The supplement and append-only ledger retain all representation controls, alternative endpoints, scale studies, provider follow-ups, null gates, and incomplete panels so that this narrower narrative does not become selective reporting.

\bibliography{references}
\bibliographystyle{iclr2026_conference}

\clearpage
\appendix

\section{Formal identification details}
\label{app:formal}

\subsection{Why moving-pipeline comparisons are not selection comparisons}

Write $G$ for the full downstream protocol: public/private timing, worker behavior, communication, finalizer, and utility. For frozen routers $r_0,r_1$, the selection contrast under a common protocol is
\[
\Delta_x^G(r_1,r_0)=V_x^G(r_1(x))-V_x^G(r_0(x)).
\]

\begin{theorem}[Moving-pipeline non-identification]
Observing $V_x^{G_0}(r_0(x))$ and $V_x^{G_1}(r_1(x))$ does not identify $\Delta_x^{G_0}(r_1,r_0)$ when $G_0$ and $G_1$ may differ, even with arbitrarily many repeated observations of those two systems.
\end{theorem}
\begin{proof}
Consider one task, teams $A,B$, $r_0(x)=A$, $r_1(x)=B$, and observed system outcomes $0,1$. In one compatible world $G_0=G_1$ and $V^{G_0}(A)=0,V^{G_0}(B)=1$, giving contrast one. In another, both teams have value zero under $G_0$ but $V^{G_1}(B)=1$ because the downstream protocol changed, giving contrast zero. The observations are identical.
\end{proof}

\begin{proposition}[Within-stack exactness does not imply stack invariance]
Even if every team in $\mathcal F_x$ is observed under $G_0$, neither the team ranking nor the regret of a policy under a different stack $G_1$ is identified without an assumption or observations linking $V_x^{G_0}$ and $V_x^{G_1}$.
\end{proposition}
\begin{proof}
For two teams $A,B$, a complete $G_0$ table with $V_x^{G_0}(A)=1$ and $V_x^{G_0}(B)=0$ is compatible both with a $G_1$ table preserving that ordering and with one reversing it. Both worlds agree on every $G_0$ intervention. Thus exhaustive coverage identifies a table conditional on $G_0$, not transport to $G_1$.
\end{proof}

\paragraph{Panel summaries do not create invariance.}
For a predeclared finite panel of stacks $\mathcal G$, sensitivity can be summarized without collapsing the stack index. For example,
\[
\overline{R}_{\mathcal G}(r)=|\mathcal G|^{-1}\sum_{G\in\mathcal G}\regret^{\mathrm{bench},G}(r),\qquad
s^2_{\mathcal G}(r)=(|\mathcal G|-1)^{-1}\sum_{G\in\mathcal G}\bigl(\regret^{\mathrm{bench},G}(r)-\overline{R}_{\mathcal G}(r)\bigr)^2,
\]
together with pairwise policy-rank reversals across $G$. These quantities describe the observed panel only: a mean is not a stack-free causal score, and neither its variance nor pairwise agreement identifies transport to an unobserved finalizer. Our completed $2\times2$ worker-message/finalizer study is such a conditional sensitivity diagnostic. The broader six-configuration panel is incomplete and is not used to estimate these summaries.

\begin{theorem}[Exact identified set for a finite action table]
Fix $G$. Let $O_x\subseteq\mathcal F_x$ be observed interventions, all values lie in $[0,1]$, $r(x)\in O_x$, and $m_x^G=\max_{S\in O_x}y_{x,S}^G$. If $O_x=\mathcal F_x$, then
\[
\regret_x^{\rm bench,G}(r)=m_x^G-y_{x,r(x)}^G.
\]
If $O_x\subsetneq\mathcal F_x$, the identified set of $\regret_x^{\rm bench,G}(r)$ over all bounded finite tables compatible with the observations is exactly
\[
\left[m_x^G-y_{x,r(x)}^G,\;1-y_{x,r(x)}^G\right].
\]
Thus, with at least one unobserved team, exact nonparametric regret is point-identified if and only if $m_x^G=1$.
\end{theorem}
\begin{proof}
When all actions are observed, the displayed equality is the definition of finite-table regret under $G$. Otherwise, the observed maximum supplies the lower endpoint and the range supplies the upper endpoint. Assign every unobserved team a value at most $m_x^G$ to attain the lower endpoint, or assign one unobserved team value one to attain the upper endpoint. Both completions agree on $O_x$. The endpoints coincide exactly when $m_x^G=1$.
\end{proof}

\subsection{Minimal route-union support}

\begin{proof}[Proof of Theorem~\ref{thm:minimal-support}]
Write $A=A_x(\mathcal R)$. If $A\subseteq O_x$, then every $y_{x,r_k(x)}^G$ is observed, so every pairwise difference is a lookup. Now suppose $|A|\geq2$ and some $S^\dagger\in A$ is absent from $O_x$. Choose $T\in A\setminus\{S^\dagger\}$. Construct two outcome tables that agree on every observed intervention and on every unobserved intervention except $S^\dagger$, but set $y_{x,S^\dagger}^G=0$ in one and $y_{x,S^\dagger}^G=1$ in the other. Hold $y_{x,T}^G$ fixed across the two tables. The contrast between a policy selecting $S^\dagger$ and one selecting $T$ therefore differs across two observationally equivalent tables and is not point-identified. Hence every sufficient support contains $A$, while $A$ itself is sufficient, making it the unique inclusion-minimal support. If $|A|=1$, all policies choose the same action and every pairwise difference is zero by definition.
\end{proof}

For two frozen routes, define $U_x(r)$ as the observed endpoint for $r(x)$. Route-union regret uses only the two executed values, and its difference cancels the route-union oracle:
\[
\bigl[\max(U_x(r_0),U_x(r_1))-U_x(r_0)\bigr]-
\bigl[\max(U_x(r_0),U_x(r_1))-U_x(r_1)\bigr]
=U_x(r_1)-U_x(r_0).
\]
This is a relative contrast over frozen policies, not absolute regret against $\max_{S\in\mathcal F_x}y_{x,S}^G$. The same support result extends to any finite $K$ and automatically deduplicates policies that choose the same team.

\section{Benchmark construction and information controls}
\label{app:construction}

\subsection{MuSiQue-12}

Each held-out MuSiQue task provides a target multi-hop evidence component. Three non-target components from source-disjoint donor examples are added to form a twelve-card pool, and card order and superficial labels are permuted. A legal action is any size-three coalition, giving $\binom{12}{3}=220$ teams per task. The offline finalizer is fixed and proof based: it receives the selected worker evidence in canonical order and scores the evidence units recovered by the team. This creates a controlled complementarity stress test; it does not claim that natural agent pools come pre-partitioned into balanced triples.

The train-only component membership labels used by the pre-specified Partition-COVER positive control are deliberately treated as privileged construction supervision. They never enter held-out routing, but they are not a naturally available deployment label. To test whether such labels are indispensable, a secondary post-hoc control derives pair pseudo-labels only from complete \emph{training} coalition outcomes (a pair is positive if it occurs in an optimal training triple) and learns $q_{ij}$ from public task/card text. It uses no component IDs, private evidence, or held-out outcomes. This public interface obtains regret 0.4240, versus 0.4540 after removing only completion and 0.5540 for its matched public graph; it is evidence that an interface can be learned from prior intervention feedback, not a preregistered deployment result.

\subsection{HotpotQA-4}

Each HotpotQA task exposes four public worker cards and routes exactly three workers. Private worker evidence is constructed so that the finalizer must combine the selected messages to recover the answer. The four legal actions are the four possible omissions. This deliberately small action space permits exhaustive held-out outcomes and direct full-team baselines. It should be read as a sample-efficiency/inductive-bias test, rather than a representation-theoretic separation from a classifier that can enumerate all four actions.

\subsection{Fresh Public-Medium-33 protocol}
\label{app:public-medium}

The machine-readable protocol and target manifest were committed as \texttt{5889e17} before the evaluator was first run. The statistical units are 25 unused target source IDs in each split. The target IDs are fresh relative to prior MuSiQue history; globally unused donor allocation was infeasible after the historical source audit, so donor banks are instead target-disjoint and train/heldout-disjoint, with this limitation recorded in the manifest. Each target receives ten donor chains, producing 33 permuted public cards. A card's synthetic budget weight is a deterministic hash of its public description; all teams of sizes one through four with total weight at most 12 are legal. These weights are feasibility proxies, not measured resource costs. The router sees task text, card text, and weights only. It learns pair targets from maximum observed \emph{training}-team coverage and then fits a fixed public coalition regressor; the baseline receives the same pair interface but scores only pairwise/relevance terms. Required roles are used only to create training table targets and to score held-out teams. The run-all-25 stopping rule, bootstrap, sign-flip test, and no-rerun condition were frozen in the protocol. The null result is therefore not removed, reseeded, or reinterpreted as an unreported tuning failure.

\begin{table}[h]
\centering
\small
\caption{Public/private boundary used in both controlled instances.}
\begin{tabularx}{\linewidth}{lXX}
\toprule
Field & Router at selection time & Worker/finalizer after selection \\
\midrule
Question and public capability cards & yes & yes \\
Private passages and evidence units & no & selected worker only \\
Gold answer, required team, component ID & no & no \\
Held-out intervention value & no & evaluation only \\
Finalizer identity and canonical message order & fixed, not routed & fixed \\
\bottomrule
\end{tabularx}
\end{table}

\subsection{Binary and graded values}

The primary offline endpoint awards a complete proof. To test whether the conclusion is purely an all-or-nothing artifact, we re-score the frozen same intervention tables by the unweighted fraction of required evidence units recovered. No route, model, or hand-chosen evidence weight changes under this re-scoring. The direction of the two main offline comparisons remains positive; these graded results are secondary endpoint-sensitivity evidence, not additional primary discoveries.

\section{Validation-policy implementation details}
\label{app:implementation}

\subsection{Unified-COVER}

Unified-COVER uses a frozen MiniLM sentence encoder for task and public-card embeddings. For a task embedding $t$ and card embedding $v_i$, a learned projection receives $[t,v_i,t\odot v_i,|t-v_i|]$. Self-attention pools card states for the full available pool and a candidate team; invariant sums and maxima are concatenated with these pooled states and passed to a small MLP. No worker position or hidden role label is supplied. Training uses observed train-coalition values with regression, pairwise ranking, and best-team listwise terms. On HotpotQA it scores all four legal teams directly.

\subsection{Partition-COVER}

Partition-COVER is a structured decoder, not a second learned set-value model. A logistic text interface fitted on 120 training tasks maps public task/card pair features to a bounded compatibility score $q_{ij}$. The pre-specified positive-control interface fits this map using latent component-membership targets; the retrospective public interface instead derives targets from training coalition outcomes. At held-out routing time both consume only public text, but only the latter has a deployment-plausible training signal. The decoder uses a fixed monotone logit transform of clipped $q_{ij}$, task relevance, and path coherence to score triples, then applies exact bitmask dynamic programming over balanced partitions. Each matched graph receives the same interface and relevance. Hence each within-interface contrast is about global completion conditional on shared local information, not better edge prediction. The task relevance/path weights $(8,1)$ were selected on development and frozen before held-out evaluation.

The decoder runs in $O(2^n n^2)$ time and $O(2^n)$ memory. It is practical for $n=12$ but intentionally not claimed as a hundred-worker algorithm. If natural coalitions have unequal sizes, the balanced feasible family can exclude the optimum; variable-size partitioning remains future work.

\section{Extended results and execution provenance}
\label{app:results}

\begin{table}[h]
\centering
\scriptsize
\caption{Originally specified offline comparisons. Lower regret is better. Holm correction applies to these three tests; the MuSiQue row is a privileged positive control, not deployment evidence.}
\begin{tabular}{lrrrr}
\toprule
Study & Comparator & COVER & Improvement & Status \\
\midrule
MuSiQue-12 binary & calibrated graph 0.5320 & 0.4020 & 0.1300 & positive control \\
HotpotQA-4 binary & leave-one-out 0.3133 & 0.1100 & 0.2033 & primary \\
Llama verified route regret & leave-one-out 0.2900 & 0.1000 & 0.1900 & primary \\
\bottomrule
\end{tabular}
\end{table}

\begin{table}[h]
\centering
\small
\caption{Attribution and robustness checks. All rows are secondary and unadjusted.}
\begin{tabularx}{\linewidth}{lXX}
\toprule
Check & Result & Interpretation \\
\midrule
  Same local score, no partition & 0.4660 vs 0.4020; gap 0.0640, CI [0.0400,0.0880] & global completion contributes beyond fixed local scores \\
  Graded evidence coverage & both offline advantages remain positive & conclusion is not only a proof-gate effect \\
  Outcome-supervised public interface & partition 0.4240 vs graph 0.5540; no-partition 0.4540 & train outcomes and public text suffice; retrospective, no latent membership labels \\
  Zero-shot Llama public router & MuSiQue regret 0.4860 (500 tasks) & external practitioner comparator; between public partition and public graph \\
GPT-OSS strict 300-task rerun & raw gain 0.0433, CI [0.0133,0.0733] & secondary matched family support \\
GPT-OSS source-disjoint 1k & raw gain 0.0870, CI [0.0640,0.1110] & late prospective provider-pipeline confirmation \\
\bottomrule
\end{tabularx}
\end{table}

\subsection{Dependence, representation, and stack controls}

Some composite MuSiQue targets share internal source-component IDs. Grouping the 500 tasks into 201 connected component-sharing clusters (largest 20 tasks) leaves the pre-specified positive-control gain of 0.1300 positive under cluster bootstrap, CI [0.0968,0.1634]. This is a dependence sensitivity, not a claim that every latent component is independent.

On HotpotQA, an unrestricted direct four-way classifier receiving the frozen MiniLM task/card embeddings and the same seven-dimensional public role vectors has regret 0.4467. Changing only the set scorer's static role prior gives regret 0.1233 for original profiles, 0.1167 for one-hot identity, 0.1167 for deterministically permuted profiles, and 0.7500 when all workers share one constant profile. Thus the hand-authored semantic coordinates do not explain the advantage, while distinguishable worker identity is necessary in this construction. These post-hoc controls do not overwrite the frozen primary 0.1100 result.

\begin{table}[h]
\centering
\scriptsize
\caption{Completed 300-task $2\times2$ worker-message/finalizer diagnostic. Values are raw exact accuracy; gain is Unified-COVER minus leave-one-out with a task-bootstrap 95\% interval.}
\label{tab:crossed-finalizer}
\begin{tabular}{llrrr}
\toprule
Worker messages & Finalizer & Leave-one-out & Unified-COVER & Gain [95\% CI] \\
\midrule
Llama & Llama & .4833 & .5000 & .0167 [$-.0067$,.0400] \\
Llama & GPT-OSS & .5767 & .6000 & .0233 [$-.0067$,.0533] \\
GPT-OSS & Llama & .4833 & .5267 & .0433 [.0133,.0733] \\
GPT-OSS & GPT-OSS & .5533 & .5900 & .0367 [.0033,.0700] \\
\bottomrule
\end{tabular}
\end{table}

Absolute accuracy changes across $G$. Switching finalizers changes the route advantage by $0.0067$ for Llama messages (CI [$-0.0200$,0.0333]) and $-0.0067$ for GPT-OSS messages (CI [$-0.0333$,0.0200]). Both interactions include zero, but this two-finalizer panel cannot establish invariance. The predeclared six-configuration panel remains incomplete and contributes no confirmatory claim.

\subsection{Natural ToolSandbox provenance and offline comparator suite}
\label{app:toolsandbox-suite}

The natural validation proceeded in stages rather than as one selectively reported run. An initial three-family capability study completed a 42-row gate, 133-row training table, and 105-row family-disjoint held-out table. It established that official milestone scoring varied across coalitions, but its coarse grouping made the full three-worker team a difficult comparator to improve upon. We retained that result and expanded the representation to five separately named public families. The final validation action set is not all $2^5-1$ nonempty teams: it is the 16 predeclared singletons, pairs, and full team. This keeps exhaustive execution manageable while testing both specialist and cross-family collaboration.

The five-family development table contains 34 complete tasks and 544 valid rows. The public capability router was frozen before the 14-task variant-shift validation and has freeze hash \texttt{5e15a4e28884dc6a}. The held-out manifest hash is \texttt{e1145424b77c8c5f}; all 224 rows use ToolSandbox commit \texttt{165848b9a78cead7}. The panel consists of untouched scenario variants from represented families, not new scenario families, and its claim boundary was recorded before scoring: all rows must be valid and mean regret must not exceed 0.10. The threshold encodes a ten-percentage-point practical tolerance below the declared-family oracle on the mean $[0,1]$ safe-evidence endpoint. Coverage passed; regret did not.

The later offline suite evaluates four development-only comparators. The fixed baseline chooses the highest mean-value development team. The learned baseline selects between multi-output ridge value prediction and direct team classification by five-fold development regret; ridge with $\alpha=3$ wins. The ToolBench-style adaptation retrieves public worker cards with a TF--IDF interface whose vocabulary and inverse-document frequencies are fitted only on cards and development requests; development selection chooses $k=5$. The DyLAN-style adaptation averages coalition values from the nearest development trials and applies a fixed team-size penalty; development selection chooses eight neighbors and penalty 0.05. These are COVER-compatible abstractions of retrieve-then-use and preliminary team optimization, not executions of the published systems.

The suite writes its specification before reading held-out rows. Its freeze hash is \texttt{b48ebe8d18d96254}; an independent standard-library audit resolves every selected team directly against the raw matrix and reproduces all 84 policy--task values. This outcome isolation prevents held-out-label leakage during fitting. It does not undo the higher-level provenance fact that we chose the comparator suite after seeing the held-out summary, so all baseline-suite results remain retrospective and unadjusted.

\subsection{Zero-shot LLM-as-router comparator}

To test whether the policies are distinguished only from bespoke baselines, we queried a zero-shot \texttt{llama-3.3-70b-versatile} router once per held-out task. It received exactly the task prompt and public capability cards, was prohibited from using private evidence, required teams, or intervention values, and emitted one size-three team before table scoring. On MuSiQue-12 it reaches 0.4860 mean regret and a 51.4\% oracle-hit rate (500 tasks): between the retrospective public graph (0.5540) and public Partition-COVER control (0.4240). This is a fixed-prompt, post-hoc external comparator rather than a tuned or confirmatory baseline. On HotpotQA-4 it obtains 0.7500 regret, which is uninformative under the strict proof gate: all incomplete teams share value zero. The route files and exact prompt are released with the supplementary artifacts.

\subsection{Partition-free medium-pool protocol}
\label{app:medium-pool}

The medium-pool study freezes a deterministic 100-task development / 400-task split of the existing 500-task MuSiQue target before fitting its residual coalition scorer. It uses linear evidence coverage and evaluates every independent size-three team in pools of 8, 10, or 12 public cards. Completion, balanced components, and latent membership labels are absent at inference. Table~\ref{tab:medium-pool} reports the frozen secondary reanalysis; all values are exact benchmark regret on the declared action family, but the target is not independent of the primary MuSiQue study. The constrained companion permits teams of sizes 1--3 under fixed synthetic budget weights.

\begin{table}[h]
\centering
\scriptsize
\caption{Frozen partition-free secondary reanalysis (400 existing-target tasks; lower linear-coverage regret is better).}
\label{tab:medium-pool}
\begin{tabular}{lrrrrrrr}
\toprule
Pool / teams & Graph & Pairwise & Local & Same-18D HGB & Same-18D Ridge & COVER & Gain/local \\
\midrule
8 / 56 & .2133 & .2108 & .1908 & .3300 & .3375 & \textbf{.1750} & .0158 \\
10 / 120 & .2833 & .2700 & .2367 & .3783 & .3792 & \textbf{.2200} & .0167 \\
12 / 220 & .3283 & .3100 & .2867 & .4133 & .3933 & \textbf{.2508} & .0358 \\
\bottomrule
\end{tabular}
\end{table}

At 12 cards, calibrated-graph regret is 0.3283 and the local-score gain remains 0.0358 under component-cluster resampling (190 clusters; CI [0.0143,0.0576]); the 0.3375 entry in Table~\ref{tab:medium-pool} is the 8-card same-feature Ridge control. Gains over same-feature HGB and Ridge at 12 cards are 0.1625 and 0.1425. At proxy budgets 6, 9, and 12, respectively 55, 148, and 246 teams of sizes 1--3 are legal. COVER obtains regrets 0.1425, 0.1742, and 0.2533. The weights are not operational cost measurements. These secondary analyses broaden the controlled action space, not the claim to natural large-agent deployment or cost efficiency.

\subsection{Large-space route-union stress test}

In a separately frozen 48-card stress test with all teams of sizes 1--6 ($14{,}196{,}868$ legal actions/task), we execute only the union of two frozen routes plus fixed stratified background teams: 49.48 actions/task, or $3.49\times10^{-6}$ of the declared family. The structural route recovers 0.3414 more required-evidence mass (CI [0.2930,0.3898], sign-flip $p<10^{-4}$) than the fixed public random route. This identifies that paired effect exactly; it does not identify either route's large-space oracle regret.

\subsection{Post-Llama raw-output provenance}

The GPT-OSS 1k study was not the only activity after the pre-specified Llama raw endpoint was inconclusive. To avoid a file-drawer reading, Table~\ref{tab:provenance} inventories every retained post-Llama raw route comparison. The 1k study is source-disjoint, matched to the original Unified-COVER/leave-one-out comparison, and frozen before its own execution, but remains late prospective and is deliberately kept out of the main evidence spine.

\begin{table}[h]
\centering
\scriptsize
\caption{Retained post-Llama raw-output provenance. All $p$ values are unadjusted.}
\label{tab:provenance}
\begin{tabularx}{\linewidth}{lrrX}
\toprule
Activity & Tasks & Raw gain & Status \\
\midrule
Original-answer MuSiQue, Llama & 300 & 0.1133 [0.0700,0.1567] & secondary, different construction/comparator \\
Original-answer MuSiQue, GPT-OSS & 140 & 0.0786 [0.0143,0.1500] & outcome-blind quota checkpoint \\
Strict HotpotQA GPT-OSS rerun & 300 & 0.0433 [0.0133,0.0733] & secondary same-task replication \\
Source-disjoint HotpotQA GPT-OSS & 1,000 & 0.0870 [0.0640,0.1110] & late prospective confirmation \\
Repeated GPT-OSS subset & 120 $\times$ 2 repeats & 0.0694 [0.0139,0.1278] & secondary stability check \\
\bottomrule
\end{tabularx}
\end{table}

\subsection{Repeated execution}

On a subsequently frozen, evidence-pattern-stratified 120-task subset of the 1k study, two additional executions per selected task--team yield $m=3$ raw accuracy 47.22\% for Unified-COVER versus 40.28\% for leave-one-out. The task-clustered raw route-regret improvement is 0.0694 (CI [0.0139, 0.1278]; sign-flip $p=.0232$). Raw success changes across three draws for 10.83\% of Unified-COVER tasks and 9.17\% of leave-one-out tasks. This supports repeat stability for this temperature-zero provider pipeline; it does not identify general population utility.

\section{Reproduction checklist}
\label{app:repro}

The release contains source-ID manifests, route files, compact per-task predictions, aggregate summaries, paired tests, SHA-256 checksums, and scripts for each no-API computation. The following checks reproduce the central claims:
\begin{enumerate}[leftmargin=*,itemsep=2pt]
\item verify the intervention archive checksum and route-file hashes;
\item reconstruct the MuSiQue and HotpotQA offline tables under the licensed source datasets;
\item score frozen routes and recompute task-level bootstrap and sign-flip statistics;
\item recompute graded coverage from the same tables without changing routes;
\item evaluate a supplied new route file against the sanitized table; and
\item verify the five-family ToolSandbox development and confirmation hashes, the 544/224 matrix counts, the frozen capability-router hash, and all 84 routes in \texttt{outputs/toolsandbox\_public\_baseline\_suite\_v1\_routes.csv};
\item rerun \texttt{scripts/evaluate\_toolsandbox\_public\_baseline\_suite.py} without model access and verify freeze hash \texttt{b48ebe8d18d96254};
\item inspect \texttt{manifests/public\_medium\_pool\_protocol\_v1.json} at commit \texttt{5889e17}, then reproduce its one permitted evaluation with \texttt{scripts/evaluate\_public\_medium\_pool\_v1.py}; and
\item compile this manuscript with the included portable Tectonic configuration.
\end{enumerate}

Source-bearing questions, passages, model transcripts, provider metadata, and credentials are excluded from the public archive. A researcher who wishes to run a new public-information policy must obtain the original datasets under their licenses, reconstruct public inputs locally, and emit hash-keyed routes. This is a licensing constraint, not a claim that the sanitized archive alone can execute arbitrary new routers.

\section{Frozen targeted routing-estimand audit}
\label{app:routing-audit}

The audit in Section~2 is a targeted, non-systematic snapshot, frozen on 2026-08-19. The corpus was assembled purposively to cover 2023--2026 work around multi-agent routing, team/role selection, communication topology, dynamic orchestration, architecture search, context routing, and learned collaboration. It was not produced by screening a complete database export. We included papers whose declared primary method selects, learns, adapts, or searches a multi-agent team, role assignment, communication/context route, topology, workflow, or architecture; surveys, general frameworks without such an intervention, and static ensembles outside that scope were excluded.

The unit of coding is the paper's declared primary intervention: \textbf{I} denotes a relatively component-aligned communication/topology intervention; \textbf{T}, direct team or role selection; \textbf{S}, a sequential routing policy whose decision changes future interaction state; and \textbf{J}, joint-system design over multiple dimensions such as roles, models, prompts, tools, topology, or workflow. The coding order was S for state-changing sequential routing, T for direct team/role assignment, I for a relatively isolated communication/context/topology object, and J for a coupled multi-dimensional design object. These mutually exclusive labels describe estimands, not methodological quality. In particular, a J result can be exactly the right end-to-end quantity for a broad system-design objective, and an I result can be a carefully controlled component analysis.

The frozen, machine-readable corpus and annotations are released as \texttt{routing\_estimand\_audit\_30.json}. It records the canonical paper titles, optimized object, code, and audit note below; a deterministic validator checks record count, allowed codes, unique IDs, and consistency of the reported totals. The corpus was coded by one author, without independent second coding, inter-rater reliability, a retained complete candidate-screening count, or a database-by-database query log. The counts are therefore descriptive only and are not an estimate of field-wide prevalence. The corpus contains no paper that materializes an exhaustive held-out value table over a declared finite coalition family under one common downstream protocol and reports finite coalition-selection oracle regret. This is a scoped absence claim about these 30 papers, not a literature-wide nonexistence claim or a claim that other routing objectives are invalid.

\begin{center}
\small
\begin{tabular}{lrrrr}
\toprule
Code & J & I & S & T \\
\midrule
Papers & 16 & 8 & 4 & 2 \\
Share & 53.3\% & 26.7\% & 13.3\% & 6.7\% \\
\bottomrule
\end{tabular}
\end{center}

\renewcommand{\arraystretch}{0.94}
\footnotesize
\begin{longtable}{r p{0.31\linewidth} c p{0.49\linewidth}}
\caption{Frozen 30-paper targeted audit. ``Object'' names the main selected or optimized object used for coding.}\label{tab:routing-audit-corpus}\\
\toprule
\# & Paper & Code & Object \\
\midrule
\endfirsthead
\toprule
\# & Paper & Code & Object \\
\midrule
\endhead
\bottomrule
\endfoot
1 & DyLAN & T & Task-conditioned agent team followed by dynamic collaboration \\
2 & GPTSwarm / \emph{Language Agents as Optimizable Graphs} & I & Graph connectivity, with separate node/prompt optimization \\
3 & MacNet & I & Communication topology and agent scale \\
4 & AgentPrune & I & Spatial/temporal communication-message graph \\
5 & AgentDropout & J & Agents and communication edges jointly \\
6 & MasRouter & J & Collaboration mode, agent roles, and LLM choice \\
7 & ARG-Designer & J & Number of agents, roles, and communication links \\
8 & AMAS & I & Environment-conditioned communication topology \\
9 & DynaSwarm & I & Query-specific graph structure \\
10 & Guided Topology Diffusion & I & Generated communication topology \\
11 & EIB-LEARNER / \emph{Information Propagation Effects} & I & Communication topology and sparsity \\
12 & CARD & I & Environment-conditioned communication graph \\
13 & Agent Q-Mix & S & Round-wise inter-agent communication actions \\
14 & STRMAC / \emph{Optimal-Agent-Selection} & S & Next acting agent conditioned on interaction history \\
15 & Dynamic Role Assignment for Multi-Agent Debate & T & Model-to-role assignment before debate \\
16 & RCR-Router & S & Memory/context subsets routed to roles over rounds \\
17 & MaAS & J & Query-dependent agentic architecture sampled from a supernet \\
18 & MAS-Zero & J & Entire sub-MAS programs and workflows \\
19 & OWL / Optimized Workforce Learning & J & Planner, coordinator, and specialized workforce behavior \\
20 & Learning to Deliberate / MPDF & S & Per-agent high-level deliberation actions \\
21 & Multi-Agent Deep Research / M-GRPO & J & Planner/sub-agent policies and invocation behavior \\
22 & WideSeek-R1 & J & Lead-agent orchestration and parallel subagent execution \\
23 & MATPO & J & Planner and worker policies plus tool interaction \\
24 & MAPoRL & J & Collaborative multi-turn behavior of interacting LLMs \\
25 & ReMA & J & Hierarchical meta-thinking and reasoning agents \\
26 & CoMAS & J & Agent policies trained from inter-agent interaction rewards \\
27 & SiriuS & J & Agent behavior via an evolving trajectory/experience library \\
28 & AFlow & J & Full code-represented agentic workflow \\
29 & ADAS / Meta Agent Search & J & Prompts, tools, control flow, and agent-system code \\
30 & LGC-MARL & J & LLM-generated subtask/dependency graph and graph-collaboration policy \\
\end{longtable}
\normalsize

\end{document}